\documentclass{article}

\usepackage[preprint]{corl_2022} 
\usepackage{algorithm}
\usepackage[noend]{algpseudocode}
\usepackage{amsthm}
\usepackage{amsmath}
\usepackage{amsfonts}
\newtheorem{assumption}{Assumption}
\newtheorem{definition}{Definition}
\newtheorem{lemma}{Lemma}
\newtheorem{theorem}{Theorem}
\usepackage{graphicx}
\usepackage{graphics}

\newcommand{\Sc}{\mathcal{S}}
\newcommand{\Ac}{\mathcal{A}}
\newcommand{\Tc}{\mathcal{T}}
\newcommand{\Tt}{T_{\max}}
\newcommand{\Dc}{\mathcal{D}}
\newcommand{\Hc}{\mathcal{H}}
\newcommand{\rc}{{\it r}}
\DeclareMathOperator*{\argmax}{arg\,max}

\title{Safe Reinforcement Learning with \\ Contrastive Risk Prediction}
\author{
  Hanping Zhang, \quad  Yuhong Guo\\
  Carleton University, Ottawa, Canada\\ 
}

\begin{document}
\maketitle

\begin{abstract}
As safety violations can lead to severe consequences in real-world robotic applications, 
the increasing deployment of Reinforcement Learning (RL) in robotic domains
has propelled the study of safe exploration for reinforcement learning (safe RL).
In this work, we propose a risk preventive training method for safe RL,
which learns a statistical contrastive classifier to predict the probability 
of a state-action pair leading to unsafe states. 
Based on the predicted risk probabilities, we can collect risk preventive trajectories
and reshape the reward function with risk penalties to induce safe RL policies. 
We conduct experiments in robotic simulation environments. 
The results show the proposed approach has comparable performance 
with the state-of-the-art model-based methods and outperforms conventional model-free safe RL approaches.
\end{abstract}


\section{Introduction}
Reinforcement Learning (RL) offers a great set of technical tools for 
robotics by 
enabling robots to automatically learn behavior policies through 
interactions with the environments
and solving the decision-making problems in robotic environments
\citep{kober2013reinforcement}. 
The deployment of RL in robotic domains can range from 
small daily used robots~\citep{mahadevan1992automatic, gullapalli1994acquiring}
to simulation robots with specific real-world applications~\citep{bagnell2001autonomous} 
and humanoid robots~\citep{schaal1996learning}. 
Conversely, the applications in real-world robotic domains also pose important new challenges for RL research. 
In particular, many real-world robotic environments and tasks, such as 
human-related robotic environments~\citep{brunke2021safe}, 
helicopter manipulation~\citep{martin2009learning, koppejan2011neuroevolutionary}, 
autonomous vehicle~\citep{wen2020safe}, and aerial delivery~\citep{faust2017automated},
have very low tolerance for violations of safety constraints, as such violation can cause severe consequences.
This raises a substantial demand for safe reinforcement learning techniques.

Safe exploration for RL (safe RL) investigates RL methodologies with critical safety considerations,
and has received increased attention from the RL research community. 
In safe RL, in addition to the reward function~\citep{sutton2018reinforcement}, 
a RL agent often deploys a cost function to maximize the discounted cumulative reward 
while satisfying the cost constraint~\citep{mihatsch2002risk, hans2008safe, ma2022conservative}. 
A comprehensive survey of safe RL categorizes the safe RL techniques into two classes: 
modification of the optimality criterion and modification of the exploration process~\citep{garcia2015comprehensive}. 
For modification of the optimality criterion, previous works mostly focus on the modification of the reward. 
Many works~\citep{ray2019benchmarking,shen2022penalized, tessler2018reward, hu2020learning, thomas2021safe,zhang2020cautious} 
pursue such modifications by shaping the reward function with penalizations induced from different 
forms of cost constraints. 
For modification of the exploration process, safe RL approaches focus on training RL agents on modified trajectory data. 
For example, some works deploy backup policies to recover from safety violations 
to safer trajectory data that satisfy the safety constraint~\citep{thananjeyan2021recovery,bastani2021safe, achiam2017constrained}.

In this paper, we propose a novel risk preventive training (RPT) method to 
tackle the safe RL problem. 
The key idea is to learn a contrastive classification model to predict the risk---the probability of 
a state-action pair leading to unsafe states,
which can then be deployed to modify both the exploration process and the optimality criterion. 
In terms of exploration process modification, we 
collect trajectory data in a risk preventive manner based on the predicted probability of risk. 
A trajectory is terminated if the next state falls into an unsafe region that has above-threshold risk values. 
Regarding optimality criterion modification, we reshape the reward function by penalizing it 
with the predicted risk for each state-action pair. 
Benefiting from the generalizability of risk prediction, 
the proposed approach can avoid safety constraint violations 
much early in the training phase and induce safe RL policies.
We conduct experiments using 
four robotic simulation environments on 
MuJoCo~\citep{todorov2012mujoco}. 
Our model-free approach produces comparable performance with a state-of-the-art model-based safe RL method
SMBPO~\citep{thomas2021safe} and greatly outperforms other model-free safe RL methods.
The main contributions of the proposed work can be summarized as follows:
\begin{itemize}
    \item 
This is the first work that simultaneously learns a contrastive classifier to perform risk prediction
while conducting safe RL exploration. 
\item
With risk prediction probabilities, the proposed approach is able to 
perform both exploration process modification through risk preventive trajectory collection
and optimality criterion modification through reward reshaping. 
\item
As a model-free method, the proposed approach achieves comparable performance to 
the state-of-the-art model-based safe RL method and outperforms other model-free methods
in robotic simulation environments. 
\end{itemize}

\section{Related Works}

\paragraph{Safe Reinforcement Learning for Robotics.}
Reinforcement Learning (RL) offers a set of great tools for robotics and 
is becoming an important part of robotic learning. 
RL studies an agent's decision-making at higher abstractive level, and enables a robot to learn an optimal behavior policy by interacting with the environment~\citep{kober2013reinforcement}. 
With the increasing application demands in robotic environments, 
new challenges are raised for reinforcement learning. 
For example, some robotic environments (in particular the
human-related environments) 
have very low tolerance for violations of safety constraints, 
where safety is one most important concern~\citep{brunke2021safe}. 
Martín H et al.~\citep{martin2009learning} and Koppejan et al.~\citep{koppejan2011neuroevolutionary} studied the RL application on helicopters based on strict safety assumptions. 
Wen et al. developed a safe RL approach called Parallel Constrained Policy Optimization (PCPO) to specifically enhance the safety of autonomous vehicles~\citep{wen2020safe}. 
Faust et al. developed a RL-based unmanned aerial vehicles (UAVs) system for aerial delivery tasks to avoid static obstacles~\citep{faust2017automated}. 
Kahn et al.~\citep{kahn2017uncertainty} propose a model-based collision avoidance approach that is applicable to robotics like a quadrotor and a RC car. 
Todorov et al.~\citep{todorov2012mujoco} developed a robotic simulation environment named MuJoCo, 
which promotes the study of the RL applications in robotic environments. 
Thomas et al. further modified the MuJoCo environment to define safety violations for robotic simulations~\citep{thomas2021safe}.

\paragraph{Safe Reinforcement Learning Methods.}
Many methods have been developed in the literature for safe RL. Altman et al. introduced the Constrained Markov Decision Process (CMDP)~\citep{altman1999constrained} to formally define the problem of safe exploration in reinforcement learning. Mihatsch et al. introduced the definition of risk into safe RL and intended to find a risk-avoiding policy based on risk-sensitive controls~\citep{mihatsch2002risk}. 
Hans et al. further differentiated the states as ``safe" and ``unsafe" states based on human-designed criteria, 
while a RL agent is considered to be not safe if it reaches "unsafe" states~\citep{hans2008safe}. 
Garc{\i} et al.~\citep{garcia2015comprehensive} presented a comprehensive survey on previous works on safe reinforcement learning, which categorizes the conventional safe RL methods into two classes: 
modification of the optimality criterion and modification of the exploration process. 
Ray et al.~\citep{ray2019benchmarking} %
presented
a benchmark to measure the performance of RL agents on safety concerned environments. 
More recently, Bastani et al~\citep{thananjeyan2021recovery} and Thananjeyan et al.~\citep{bastani2021safe} 
focused on using backup policies of the safe regions, aiming to avoid safety violations.  
Tessler et al. applied the reward shaping technique in safe RL to penalize the normal training policy, 
which is known as Reward Constrained Policy Optimization (RCPO)~\citep{tessler2018reward}. 
Ma et al. propose a model-based Conservative and Adaptive Penalty (CAP) approach to explore safely 
by modifying the penalty in the training process~\citep{ma2022conservative}.
Zhang et al. develop a reward shaping approach built upon 
Probabilistic Ensembles with Trajectory Sampling (PETS)
~\citep{zhang2020cautious}.
It pretrains a predictor of the unsafe state in an offline sandbox environment and penalizes the reward of PETS in the adaptation with online environments.
A similar work by Thomas et al~\citep{thomas2021safe} reshapes reward functions using a model-based predictor. 
It regards unsafe states as absorbing states and trains the RL agent with a penalized reward 
to avoid the visited unsafe states.

\section{Preliminary}
Reinforcement learning (RL) has been broadly used to train robotic agents 
by maximizing the discounted cumulative rewards. 
The representation of a reinforcement learning problem can be formulated as a Markov Decision Process (MDP)  
$M=(\Sc, \Ac, \Tc, r, \gamma)$ 
~\citep{sutton2018reinforcement}, 
where $\Sc$ is the state space for all observations, 
$\Ac$ is the action space for available actions, 
$\Tc: \Sc\times \Ac\to \Sc$ is the transition dynamics, 
$r: \Sc\times \Ac\to [r_{min}, r_{max}]$ is the reward function, 
and $\gamma\in (0, 1)$ is the discount factor. 
An agent can start from a random initial state $s_0$ to take actions and interact with the MDP environment
by receiving rewards for each action and moving to new states. 
Such interactions can produce a transition $(s_t, a_t, r_t, s_{t+1})$ at each time-step $t$
with 
$s_{t+1}=\Tc(s_t,a_t)$ and  
$r_t=\rc(s_t, a_t)$,
while a sequence of transitions comprise a trajectory 
$\tau=(s_0, a_0, r_0, s_1, a_1, r_1, \cdots, s_{|\tau|+1})$. 
The goal of RL is to learn an optimal policy $\pi^*: \Sc \to \Ac$ 
that can maximize the expected discounted cumulative reward (return): 
$\pi^*=\argmax_\pi\; J_r(\pi)=\mathbb{E}_{\tau\sim\Dc_\pi}[\sum_{t=0}^{|\tau|} \gamma^t r_t]$

\subsection{Safe Exploration for Reinforcement Learning}
Safe exploration for Reinforcement Learning (safe RL)
studies RL with critical safety considerations.
For a safe RL environment, in addition to the reward function, 
a cost function can also exist to reflect the risky status of each exploration step. 
The process of safe RL can be formulated as a Constrained Markov Decision Process (CMDP)~\citep{altman1999constrained},
$\hat{M}=(\Sc, \Ac, \Tc, r, \gamma, c, d)$,
which introduces an extra cost function $c$ and a cost threshold $d$ into MDP. 
An exploration trajectory under CMDP can be written as 
$\tau=(s_0, a_0, r_0, c_0, s_1, a_1, r_1, c_1, \cdots, s_{|\tau|+1})$,
where the transition at time-step $t$ is $(s_t, a_t, s_{t+1}, r_t, c_t)$, with a cost value $c_t$ 
induced from the cost function $c_t = c(s_t, a_t)$.
CMDP monitors the safe exploration process by requiring the cumulative cost $J_c(\pi)$ 
does not exceed the cost threshold $d$,
where $J_c(\pi)$ can be defined as the expected total cost of the exploration, 
$J_c(\pi)=\mathbb{E}_{\tau\sim \Dc_\pi}[\sum_{t=0}^{|\tau|} c_t]$ ~\citep{ray2019benchmarking}. 
Safe RL hence aims to learn 
an optimal policy $\pi^*$
that can maximize the expected discounted cumulative reward 
subjecting to a cost constraint, as follows: 
\begin{align}
\label{equation:cmdp}
	\pi^*=\argmax_\pi J_r(\pi)=\mathbb{E}_{\tau\sim \Dc_\pi}\left[\sum\nolimits_{t=0}^{|\tau|} \gamma^t r_t\right],\quad 
	\text{s.t. }\, J_c(\pi)=\mathbb{E}_{\tau\sim \Dc_\pi}\left[\sum\nolimits_{t=0}^{|\tau|} c_t\right]
	\leq d
\end{align}

\section{Method}
Robot operations typically have low tolerance for risky/unsafe states and actions,
since a robot could be severely damaged in real-world environments 
when the safety constraint being violated. 
Similar to~\citep{hans2008safe}, 
we adopt a strict setting in this work for the safety constraint
such that any ``unsafe" state can cause violation of the safety constraint
and the RL agent will terminate an exploration trajectory when encountering an ``unsafe" state.
In particular, we have the following definition: 
\begin{definition}
\label{definition:unsafe}
For a state $s$ and an action $a$, the value of the cost function $c(s,a)$ 
can either be $0$ or $1$. 
When $c(s,a)=0$, the induced state $\Tc(s,a)$ is defined as a safe state;
when $c(s,a)=1$, the induced state $\Tc(s,a)$ is defined as an unsafe state,
which triggers the violation of safety constraint and hence causes the termination of the trajectory. 
\end{definition}

Based on this definition, the cost threshold $d$ in Eq.~(\ref{equation:cmdp}) should be set strictly to $0$.
The agent is expected to learn a safe policy $\pi$ 
that can operate with successful trajectories containing only safe states. 
Towards this goal, we propose a novel risk prediction method for safe RL.
The proposed method deploys a contrastive classifier to predict the probability 
of a state-action pair leading to unsafe states, which can be trained 
during the exploration process of RL and generalized to previously unseen states.  
With risk prediction probabilities, a more informative cumulative cost $J_c(\pi)$ 
can be formed to prevent unsafe trajectories and 
reshape the reward in each transition of a trajectory 
to induce safe RL policies. 
Previous safe RL methods in the literature can typically be categories 
into two classes: modification of the optimality criterion and modification of the exploration process 
\citep{garcia2015comprehensive}. 
With safety constraints and risk predictions, the proposed approach (to be elaborated below) 
is expected to incorporate the strengths of both categories of safe RL techniques.

\subsection{Risk Prediction with Contrastive Classification }
\label{section:risk}
Although a RL agent 
would inevitably encounter unsafe states during the initial stage of
the exploration process in an unknown environment,
we aim to quickly learn from the unsafe experience through statistical learning
and generalize the recognition of unsafe trajectories to 
prevent risk for future exploration. 
Specifically, we aim to compute the probability of a state-action pair leading 
to unsafe states, i.e., $p(y=1|s_t,a_t)$, where $y\in\{0,1\}$ denotes 
a random variable 
that indicates 
whether $(s_t, a_t)$ leads to an unsafe state $s_u\in S_U$.
The set of unsafe states, $S_U$, can be either pre-given or collected during initial exploration. 
However, directly training a binary classifier to make such predictions
is impractical as it is difficult to judge whether a state-action pair 
is {\em safe}---i.e., never leading to unsafe states. 
For this purpose, we propose to 
train a contrastive classifier $F_\theta(s_t,a_t)$ with model parameter $\theta$ 
to discriminate a positive state-action pair $(s_t,a_t)$ in a trajectory that leads to unsafe states (unsafe trajectory) 
and a random state-action pair from the overall distribution of any trajectory.
Such a contrastive form of learning can conveniently avoid the identification problem
of absolute negative (safe) state-action pairs.

Let $p(s_t,a_t|y=1)$ denote 
the presence probability of a state-action pair $(s_t,a_t)$ in a trajectory that leads to unsafe states,
and $p(y=1)$ denote the distribution probability of unsafe trajectory in the environment. 
The contrastive classifier $F_\theta(s_t,a_t)$ can be expressed as:
\begin{equation}
\label{equation:risk1}
F_\theta(s_t,a_t)=\frac{p(s_t,a_t|y=1)p(y=1)}{p(s_t,a_t|y=1)p(y=1)+p(s_t,a_t)}
\end{equation}
which contrastively identifies the state-action pairs in unsafe trajectories 
from pairs in the overall distribution.

From the definition of $F_\theta(s_t,a_t)$ in Eq.(\ref{equation:risk1}),
one can easily derive the probability of interest, $p(y=1|s_t,a_t)$, 
by Bayes' theorem, as follows:
\begin{equation}
\begin{aligned}
\label{equation:risk2}
 p(y=1|s_t,a_t)&= \frac{p(s_t,a_t|y=1)p(y=1)}{p(s_t,a_t)} 
	= \frac{F_\theta(s_t,a_t)}{1-F_\theta(s_t,a_t)}	
\end{aligned}
\end{equation}
Although the normal output range for the probabilistic classifier 
$F_\theta(s_t,a_t)$ should be $[0,1]$,
this could lead to unbounded $p(y=1|s_t,a_t)\in[0,\infty]$ through Eq.(\ref{equation:risk2}).
Hence we propose to 
rescale the output of classifier $F_\theta(s_t,a_t)$ to the range of $[0, 0.5]$.

We optimize the contrastive classifier's parameter $\theta$ using maximum likelihood estimation (MLE). 
The log-likelihood objective function can be written as:
\begin{equation}
\begin{aligned}
\label{equation:risk3}
	L(\theta)=\mathbb{E}_{p(s_t,a_t|y=1)p(y=1)}[\log F_\theta(s_t,a_t)]+\mathbb{E}_{p(s_t,a_t)}[\log(1-F_\theta(s_t,a_t))]
\end{aligned}
\end{equation}
During the training process, the positive state-action pair $(s_t,a_t)$ 
for the first term of this objective can be sampled from the observed unsafe examples, while the state-action pair $(s_t,a_t)$ for the second term can be sampled
from the overall distribution.

\subsection{Risk Preventive Trajectory}
\label{section:trajectory}

\begin{figure}
\centering
\setlength{\abovecaptionskip}{0.cm}
\includegraphics[width=0.5\textwidth]{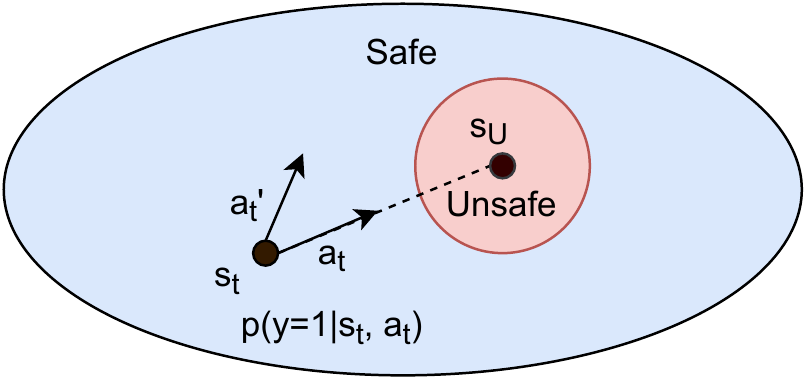}
\caption{An overview of the risk preventive process. The blue area denotes the safe region of the trajectory. The dot $s_U$ is an unsafe state. The red area around the unsafe state $s_U$ denotes the unsafe region, bounded by the value of risk prediction $p(y=1|s_t,a_t).$ The RL agent explores risk preventive trajectories by avoiding entering the unsafe region in the training process.}
\label{fig:figure}
\end{figure}

Based on Definition~\ref{definition:unsafe}, a trajectory terminates when the RL agent encounters an unsafe state 
and triggers safety constraint violation. 
It is however desirable to minimize the number of such safety violations even during the policy training process 
and learn a good policy in safe regions. 
The risk prediction classifier we proposed above provides a convenient tool for this purpose 
by predicting the probability of a state-action pair leading to unsafe states, $p(y=1|s_t, a_t)$.
Based on this risk prediction, we have the following definition for unsafe regions:
\begin{definition}
\label{definition:region}
A state-action pair $(s_t,a_t)$ falls into an \textbf{unsafe region} if the probability of $(s_t,a_t)$ 
leading to unsafe states is greater than a threshold $\eta$: $p(y=1|s_t,a_t)>\eta$, where $\eta\in(0, 1)$.
\end{definition}
With this definition, a RL agent can pursue risk preventive trajectories to avoid safety violations by staying away from unsafe regions.
Specifically, we can terminate a trajectory before violating the safety constraint
by judging the risk---the probability of $p(y=1|s_t,a_t)$. 
The process is illustrated in Figure~\ref{fig:figure}.

Without a doubt, the threshold $\eta$ is a key for determining the length $T=|\hat{\tau}|$ of an early stopped risk preventive trajectory $\hat{\tau}$. 
To approximate a derivable relation between $\eta$ and $T$, 
we make the following assumption and lemma:
\begin{assumption}
\label{assumption:length}
For a trajectory $\tau=\{s_0,a_0,r_0,c_0, s_1, \cdots, s_H\}$ that leads to an unsafe state $s_H\in S_U$, 
the risk prediction probability $p(y=1|s_t,a_t)$ increases linearly along the time steps. 
\end{assumption}

\begin{lemma}
\label{lemma:length}
Assume that Assumption~\ref{assumption:length} holds. 
Let $H\in \mathbb{N}$ denote the length of an unsafe trajectory  $\tau=\{s_0,a_0,r_0,c_0, s_1, a_1, r_1, c_1, \cdots, s_H\}$ 
that terminates at an unsafe state $s_H\in S_U$.
The number of transition steps, $T$, along this trajectory to the unsafe region determined by $\eta$ in Definition~\ref{definition:region}
can be approximated as: 
	$T=\lfloor\frac{\eta-p_0}{1-p_0}H\rfloor$, where $p_0 = p(y=1|s_0, a_0)$.
\end{lemma}
\begin{proof}
According to assumption~\ref{assumption:length}, the probability for $(s_t,a_t)$ leading to unsafe states $p(y_t=1|s_t,a_t)$ increases linearly along timesteps. 
Let $p_0$ denote the probability starting from the initial state $s_0$: $p_0=p(y_t=1|s_0,a_0)$. 
For a probability threshold $\eta\in (0,1)$ for unsafe region identification in Definition 2, 
the ratio between 
	the number of environment transition steps $T$ to the unsafe region
	and the unsafe trajectory length $H$ will be approximately (due to integer requirements over $T$)
	equal to the ratio between the probability differences of $\eta-p_0$ and $1-p_0$.
	That is, $\frac{T}{H}\approx \frac{\eta-p_0}{1-p_0}$.
Hence $T$ 
	can be approximated as:
$T=\lfloor\frac{\eta-p_0}{1-p_0}H\rfloor$.
\end{proof}
With Assumption~\ref{assumption:length}, the proof of Lemma~\ref{lemma:length} is straightforward. 
This Lemma clearly indicates that a larger $\eta$ value will allow more effective explorations with longer trajectories,
but also tighten the unsafe region and increase the possibility of violating safety constraints. 

\subsection{Risk Preventive Reward Shaping}
\label{section:reward}
With Definition~\ref{definition:unsafe}, 
the safe RL formulation in Eq.~(\ref{equation:cmdp}) can hardly induce a safe policy 
since there are no intermediate costs before encountering an unsafe state. 
With the risk prediction classifier proposed above, we can rectify this drawback
by defining the cumulative cost function $J_c(\pi)$ using the risk prediction probabilities, $p(y=1|s_t, a_t)$,
over all encountered state-action pairs. 
Specifically, we adopt a reward-like discounted cumulative cost as follows: 
$J_c(\pi)=\mathbb{E}_{\tau\sim \Dc_\pi}\left[\sum_{t=0}^{|\tau|} \gamma^t p(y=1|s_t,a_t)\right]$,
which uses the predicted risk at each time-step as the estimate cost. 
Moreover, instead of solving safe RL as a constrained discounted cumulative reward maximization problem,
we propose using Lagrangian relaxation~\citep{bertsekas1997nonlinear} 
to convert the constrained maximization CMDP problem in Eq.~(\ref{equation:cmdp}) 
to an unconstrained optimization problem,
which is equivalent to shaping the reward function with risk penalties: 
\begin{align}
	&\min_{\lambda\geq0}\max_\pi\quad \left[J_r(\pi)-\lambda (J_c(\pi)-d)\right]\\
	\Longleftrightarrow \quad &
	\min_{\lambda\geq0}\max_\pi\quad \left[J_r(\pi)-\lambda J_c(\pi)\right] \\
	\Longleftrightarrow \quad &
	\min_{\lambda\geq0}\max_\pi\quad \mathbb{E}_{\tau\sim \Dc_\pi}\left[\sum\nolimits_{t=0}^{|\tau|} \gamma^t (r_t-\lambda p(y=1|s_t,a_t))\right]
	\label{Jhatr}	
\end{align}
The Lagrangian dual variable $\lambda$ controls the degree of reward shaping with the predicted risk value. 
\begin{theorem}
\label{theorem:risk}
To prevent the RL agent from falling into known unsafe states, 
the penalty factor (i.e., the dual variable) $\lambda$ for the shaped reward
$\hat{r}_t=r_t-\lambda p(y_t=1|s_t,a_t)$ should have the following lower bound,
where $H$ and $\eta$ are same as in Lemma~\ref{lemma:length}:
\begin{equation}
\label{equation:theorem}
\lambda > \frac{(1-\gamma^{H})(r_{max}-r_{min})}
{\eta\gamma^{\lfloor\frac{\eta-p_0}{1-p_0}H\rfloor}(1-\gamma^{H-\lfloor\frac{\eta-p_0}{1-p_0}H\rfloor})}
\end{equation}
\end{theorem}
\begin{proof}
For a trajectory with length $H$ that leads to an unsafe state $s_H\in S_U$, 
the largest penalized return of the unsafe trajectory should be smaller than the 
lowest possible unpenalized return from any safe trajectory with the same length:
\begin{align}
\label{equation:reward1}
	\sup_\tau\; \left[\sum\nolimits_{t=0}^{H-1}\gamma^t (r-\lambda p_t)\right] < 
	\inf_\tau \;\left[ \sum\nolimits_{t=0}^{H-1}\gamma^t r\right]
\end{align}
where $p_t = p(y=1|s_t,a_t)$.
With this requirement, the RL agent can learn a policy 
to explore safe states and prevent the agent from entering unsafe states $S_U$ through unsafe regions. 
As the reward function $r$ is bounded within $[r_{min}, r_{max}]$, 
we can further simplify the inequality in Eq.~(\ref{equation:reward1})
by seeking the supremum of its left-hand side (LHS) and the infimum of its right-hand side (RHS). 

Based on Definition~\ref{definition:region} and Lemma~\ref{lemma:length}, we can split an unsafe trajectory $\tau$
with length $H$ into two sub-trajectories: 
a sub-trajectory within the safe region, $\tau_1=(s_0,a_0,r_0,c_0,\cdots, s_{T-1},a_{T-1}, r_{T-1}, c_{T-1})$,
where $p_t\leq \eta$
and a sub-trajectory 
$\tau_2=(s_{T},a_{T}, r_{T}, c_{T}, \cdots, s_H)$ 
after entering the unsafe region determined by $p_t > \eta$.
With Assumption~\ref{assumption:length}, we have $p_t > \eta$ for all state-action pairs in the sub-trajectory $\tau_2$. 
Then the LHS and RHS of Eq.~(\ref{equation:reward1}) can be upper bounded and lower bounded respectively as follows:
\begin{align}
\label{equation:LHSUP}
	LHS &\,\leq\,  
	\sum\nolimits_{t=0}^{H-1}\gamma^t r_{max} 
	-\left(\sum\nolimits_{t=0}^{T-1}\gamma^t \lambda \cdot 0 + \sum\nolimits_{t=T}^{H-1}\gamma^t \lambda \eta\right)
	\\
	RHS &\,\geq \, 
	\sum\nolimits_{t=0}^{H-1}\gamma^t r_{min} 
\end{align}
To ensure the satisfaction of the requirement in Eq.~(\ref{equation:reward1}), we then enforce the follows: 
\begin{align}
&\sum\nolimits_{t=0}^{H-1}\gamma^t r_{max} 
-\left(\sum\nolimits_{t=0}^{T-1}\gamma^t \lambda \cdot 0 + \sum\nolimits_{t=T}^{H-1}\gamma^t \lambda \eta\right)
< 
\sum\nolimits_{t=0}^{H-1}\gamma^t r_{min} 
\\
\Longleftrightarrow \quad &
\frac{(1-\gamma^{H})r_{max}}{1-\gamma} -\frac{\lambda\eta\cdot\gamma^{T}(1-\gamma^{H-T})}{1-\gamma} 
	< \frac{(1-\gamma^{H})r_{min}}{1-\gamma}
\\	
\Longleftrightarrow \quad &
\lambda > \frac{(1-\gamma^{H})(r_{max}-r_{min})}{\eta\gamma^{T}(1-\gamma^{H-T})}
\label{equation:reward2}
\end{align}
With Assumption~\ref{assumption:length} and Lemma~\ref{lemma:length}, 
we can estimate the length $T$ for the safe sub-trajectory as $T=\lfloor\frac{\eta-p_0}{1-p_0}H\rfloor$. 
With this estimation, the lower bound for $\lambda$ in Eq.~(\ref{equation:theorem}) can be derived. 
\end{proof}

\begin{algorithm}
\caption{Risk Preventive Training}
\label{algorithm:RPT}
\hspace*{\algorithmicindent}
\textbf{Input} 
Initial policy $\pi_\phi$,\; classifier $F_\theta$,\; trajectory set $D=\emptyset$, \;
	set of unsafe states $S_U$, \\ 
\hspace*{\algorithmicindent}\quad\qquad
	threshold $\eta$, \; penalty factor $\lambda$, \; set of unsafe trajectory length $\Hc=\emptyset$ \\
\hspace*{\algorithmicindent}
\textbf{Output} Trained policy $\pi_\phi$
\begin{algorithmic}[1]
\For{$k=1,2,...,K$}
\For{$t=1,2,...,\Tt$}
\State Sample transition $(s_t,a_t,r_t,c_t,s_{t+1})$ from the environment with policy $\pi_\phi$.
\If{$c_t>0$}
\State Add risky state-action $(s_t,a_t)$ into the unsafe state set $S_U$.
\State Add length $t$ to $\Hc$. 
Increase $\lambda$ if the lower bound increases with 
	$H=t$ and  Eq.~(\ref{equation:theorem}).
\State Stop trajectory and break.
\EndIf
\State Sample next action $a_{t+1}$ based on policy $\pi$ and next state $s_{t+1}$: $a_{t+1}=\pi_\phi(\cdot|s_{t+1})$.
	\State Calculate $p_t$ and $p_{t+1}$ using Eq.~(\ref{equation:risk2})
\State Penalize reward $r_t$ with $p_t$:\quad  $\hat{r_t}=r_t-\lambda p_t$
\State Add transition to the trajectory set: $D=D\cup (s_t,a_t,\hat{r}_t,s_{t+1})$
\If{$p_{t+1}> \eta$}
\State Stop trajectory and break.
\EndIf
\EndFor
\State Sample risky state-action pairs from $S_U$ 
\State Sample transitions from history: $(s_t,a_t,\hat{r}_t,s_{t+1})\sim D$
\State Update classifier $F_\theta$ by maximizing the likelihood $L(\theta)$ in Eq~(\ref{equation:risk3})
	\State Update policy $\pi_\phi$ with the shaped rewards $J_{\hat{r}}(\pi)$ in Eq~(\ref{Jhatr})
\EndFor
\end{algorithmic}
\end{algorithm}
\subsection{Risk Preventive Training Algorithm}
Our overall risk preventive RL training procedure is presented in Algorithm~\ref{algorithm:RPT},
which simultaneously trains the risk prediction classifier
and performs reinforcement learning with risk preventive trajectory exploration and risk preventive reward shaping.

\section{Experiment}
We conducted experiments 
on four robotic simulation environments based on the MuJoCo simulator~\citep{todorov2012mujoco}. 
In this section, we report the experimental setting and empirical results.

\begin{figure}[t]
\centering
\setlength{\abovecaptionskip}{0.cm}
\includegraphics[width=0.80\textwidth]{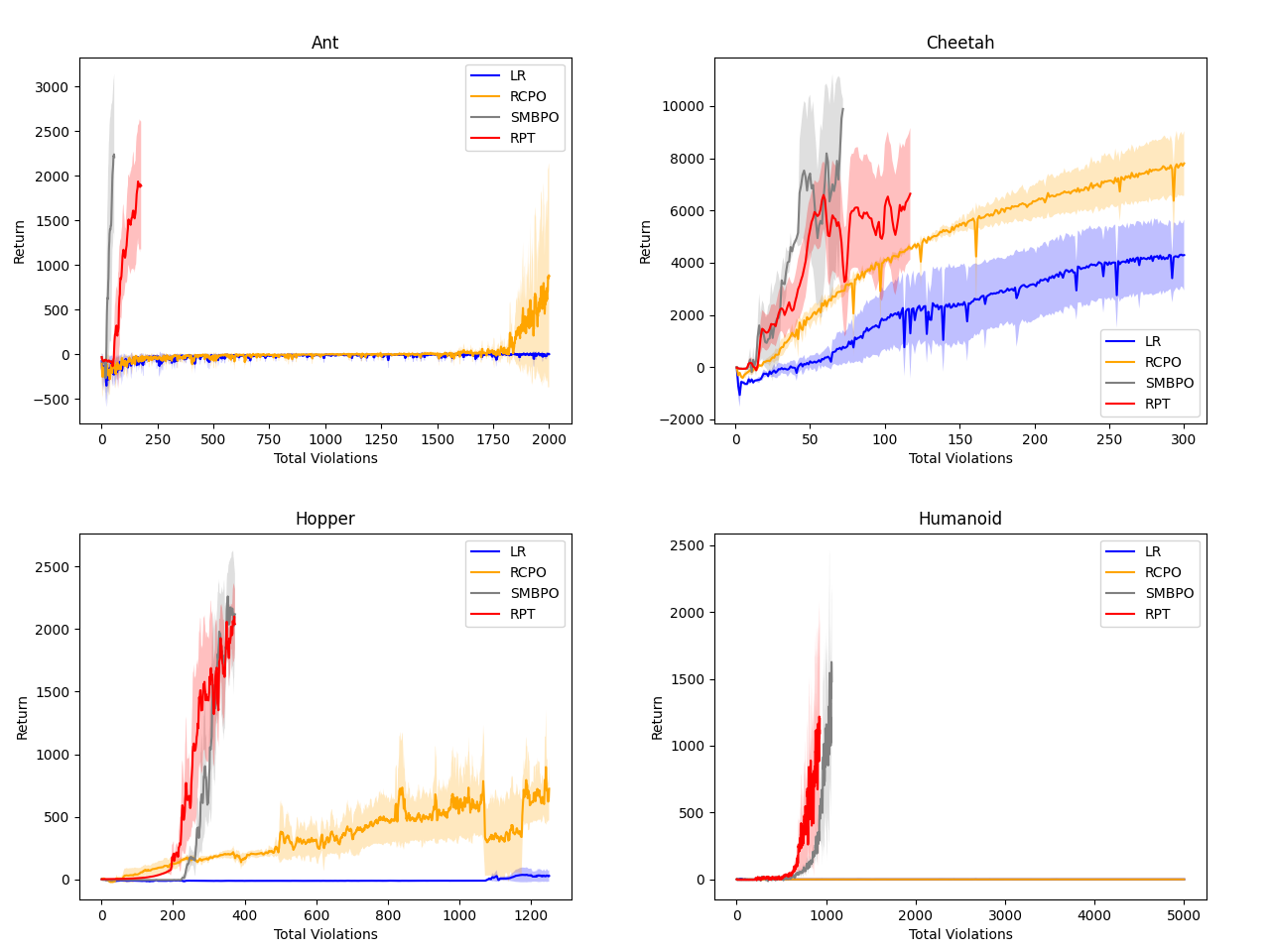}
\caption{For each method, each plot presents the undiscounted return vs. the total number of violations. 
	Results are collected for three runs with random seeds. 
	The curve shows the mean of the return over the three runs, while the shadow shows the standard deviation.}
\label{fig:plot}
\vskip -.1in
\end{figure}

\subsection{Experimental Settings}
\paragraph{Experimental Environments}
Following the experimental setting in~\citep{thomas2021safe}, 
we adopted four robotics simulation environments, {\em Ant, Cheetah, Hopper,} and {\em Humanoid}, 
using the MuJoCo simulator~\citep{todorov2012mujoco}. 
In MuJoCo environments, 
a violation is presented when the robot enters an unsafe state. 
For {\em Ant, Cheetah,} and {\em Hopper}, a robot violates the safety constraint when it falls over. 
For {\em Humanoid}, the human-like robot violates the safety constraint when the head of the robot falls to the ground. 
The RL agent reaches the end of the trajectory once it encounters the safety violation.

\paragraph{Comparison Methods}
We compare our proposed Risk Preventive Training (RPT) approach 
with three state-of-the-art safe RL methods: 
SMBPO~\citep{thomas2021safe}, RCPO~\citep{tessler2018reward}, and LR~\citep{ray2019benchmarking}.

\begin{itemize}
    \item \textbf{Safe Model-Based Policy Optimization (SMBPO)}: 
	   This is a model-based method that uses an ensemble of Gaussian dynamics based transition models.
	   Based on the transition models, it penalizes unsafe trajectories and avoid unsafe states under certain assumptions. 
    \item \textbf{Reward Constrained Policy Optimization (RCPO)}: 
	    This is a policy gradient method based on penalized reward under safety constraints. 
    \item \textbf{Lagrangian relaxation (LR)}: 
	    Its uses Lagrangian relaxation for the safety constrained RL. 
\end{itemize}

\paragraph{Implementation Details}
Implementations for LR and RCPO algorithms are adapted from the recovery RL paper~\citep{thananjeyan2021recovery}.  
For fair comparisons, following the original setting of LR and RCPO, we disable the recovery policy of the recovery RL framework, which collects offline data to pretrain the agent. 
For MBPO, we adopted the original implementation from the MBPO paper~\citep{thomas2021safe}.
Both SMBPO and RCPO are built on top of the Soft-Actor-Critic (SAC) RL method~\citep{haarnoja2018soft}. 
In the experiments, we also implemented the proposed RPT approach on top of SAC,
although RPT is a general safe RL methodology. 
We used $0.9$ as the threshold $\eta$ for risk preventive trajectory exploration. 
We collected the mean episode return and violation for $10^6$ time-steps.

\subsection{Experimental Results}

We compared all the four methods by running each method three times in each of the four MuJoCo environments. 
The performance of each method is evaluated by presenting
the corresponding return vs. the total number of violations obtained in the training process. 
The results for all the four methods (LR, RCPO, SMBPO, and RPT) are presented in Figure~\ref{fig:plot},
one plot for each robotic simulation environment. 
The curve for each method shows the learning ability of the RL agent with limited safety violations. 
From the plots, we can see both {\em RPT} and {\em SMBPO} 
achieve large returns with a small number of violations on all the four robotic tasks, 
and largely outperform
the other two methods, {\em RCPO} and {\em LR}, which have much smaller returns even 
with large numbers of safety violations. 
The proposed {\em RPT} produces slightly inferior performance than {\em SMBPO} on {\em Ant} and {\em Cheetah},
where our method requires more examples of unsafe states to yield good performance at the initial training stage. 
Nevertheless, {\em RPT} outperforms {\em SMBPO} on both {\em Hopper} and {\em Humanoid}.
As a model-free safe RL method, {\em RPT} produces an overall performance with 
the model-based method {\em SMBPO}.

\begin{figure}[t]
\vskip -.1in	
\centering
\setlength{\abovecaptionskip}{0.cm}
\includegraphics[width=0.6\textwidth]{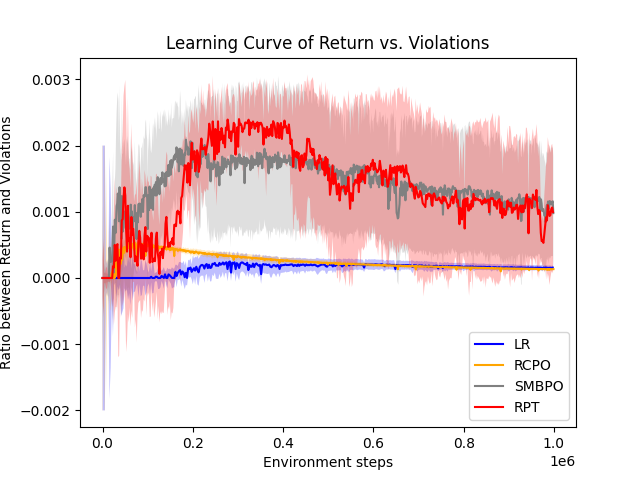}
\caption{The plot reports the ratio of the normalized return over the number of violations vs. time-steps.
	The curve for each method is calculated by averaging across all 4 environments (Ant, Cheetah, Hopper, Humanoid).}
\label{fig:plot_ratio}
\vskip -.2in	
\end{figure}

To compare the overall performance of the comparison methods across all the four environments, 
we need to take the average of the results on all four environments 
using an environment-independent measure. 
We propose to use the measure that divides the ratio between the return and the number of violations 
by the maximum trajectory return---i.e., the ratio between the trajectory normalized return
and the number of violations, and report the average performance of each method under this measure
across all four environments. 
The results are presented in Figure~\ref{fig:plot_ratio}.
We can see {\em RPT} produces comparable or even better performance on certain region of time-steps than {\em SMBPO},
and greatly outperforms {\em RCPO} and {\em LR}.
This again demonstrates that the proposed {\em RPT} produces the state-of-the-art performance for safe RL
on robotic environments.

\section{Conclusion}
Inspired by the increasing demands for safe exploration of Reinforcement Learning in Robotics, 
we proposed a novel mode-free risk preventive training method to perform safe RL 
by learning a statistical contrastive classifier to predict the probability of a state-action pair 
leading to unsafe states. 
Based on risk prediction, we can collect risk preventive trajectories that terminate early 
without triggering safety constraint violations. 
Moreover, the predicted risk probabilities are also used as penalties 
to perform reward shaping for learning safe RL policies,
with the goal of maximizing the expected return while minimizing the number of safety violations. 
We compared the proposed approach with a few state-of-the-art safe RL methods 
using four robotic simulation environments. 
The proposed approach demonstrates comparable performance with the state-of-the-art model-based method
and outperforms the model-free safe RL methods.


\bibliography{example}  

\end{document}